\documentclass[11pt,a4paper]{article}

\usepackage{fancyhdr}
\usepackage{amsmath,mathtools}
\usepackage{bm}
\usepackage{esint}
\usepackage{amsfonts}
\usepackage{amsthm}
\usepackage{amssymb}
\usepackage{amsmath}
\usepackage{amsbsy}
\usepackage{verbatim}
\usepackage{graphicx}
\usepackage{subcaption}
\usepackage{url}
\usepackage{mathrsfs}
\usepackage[hmargin = 1in,vmargin=.75in]{geometry}
\usepackage{color}
\usepackage{amstext}
\usepackage{stmaryrd}
\usepackage{enumerate}
\usepackage{algpseudocode}
\usepackage{algorithm}
\usepackage{pifont}
\usepackage{prettyref}
\usepackage{subcaption}
\usepackage[dvipsnames]{xcolor}

\usepackage{graphicx}
\usepackage{multicol}
\usepackage{amsmath}
\usepackage{bbm}
\usepackage{amssymb}
\usepackage[shortlabels]{enumitem}
\usepackage{amsthm}

\font\msbm=msbm10

\numberwithin{equation}{section}

\theoremstyle{plain}
\newtheorem{theorem}{Theorem}[section]
\newtheorem{lemma}[theorem]{Lemma}
\newtheorem{corollary}[theorem]{Corollary}
\newtheorem{example}[theorem]{Example}

\def\mathbb#1{\hbox{\msbm{#1}}}

\newcommand{\tr}{\operatorname{Tr}}

\newcommand{\lp}{\left(} 
\newcommand{\rp}{\right)} 
\newcommand{\lc}{\left\{} 
\newcommand{\rc}{\right\}} 
\newcommand{\abs}[1]{ \left|  #1 \right| }

\newcommand{\ratiocut}[1]{\text{RatioCut}\lp#1\rp}
\newcommand{\kpartition}{\lc V_i\rc_{i=1}^k}
\newcommand{\kpartitionnew}{\lc V^{(j)}\rc_{j=1}^k}
\newcommand{\cut}[1]{\text{Cut}\lp#1\rp}
\newcommand{\ones}{\mathbbm{1} }
\DeclareMathOperator{\for}{for}
\DeclareMathOperator{\gap}{gap}
\DeclareMathOperator{\sep}{sep}
\DeclareMathOperator{\ran}{ran}
\DeclareMathOperator{\dist}{dist}
\newcommand{\Liso}{L_{\text{iso}}}
\newcommand{\Wiso}{W_{\text{iso}}}
\newcommand{\Diso}{D_{\text{iso}}}
\newcommand{\Uiso}{U_{\text{iso}}}
\newcommand{\Ld}{L_{\delta}}
\newcommand{\Dd}{D_{\delta}}
\newcommand{\Wd}{W_{\delta}}
\newcommand{\norm}[1]{\left|\left| #1 \right|\right|_2}
\newcommand{\norminf}[1]{\left|\left| #1 \right|\right|_{\infty}}
\newcommand{\normone}[1]{\left|\left| #1 \right|\right|_{1}}
\newcommand{\normfro}[1]{\left|\left| #1 \right|\right|_{F}}
\newcommand{\normtwotoinf}[1]{\left|\left| #1 \right|\right|_{2,\infty}}
\newcommand{\twotoinf}{2,\infty}
\newcommand{\ddi}{d_\delta^{(i)}}
\newcommand{\inner}[2]{\left\langle #1,#2 \right\rangle }

\newcommand{\mij}{m_i^{(j)}}
\renewcommand{\ni}{n_i}
\newcommand{\nj}{n^{(j)}}
\newcommand{\sumi}{\sum_i}
\newcommand{\sumj}{\sum_j}
\newcommand{\sumip}{\sum_{i'}}
\newcommand{\sumjp}{\sum_{j'}}

\newcommand{\sumjpneqj}{\sum_{j'\neq j}}
\newcommand{\sumipneqi}{\sum_{i'\neq i}}
\newcommand{\Aiipjjp}{W_{i,i'}^{(j,j')}}
\newcommand{\Aiijjp}{W_{i,i}^{(j,j')}}
\newcommand{\Aiipjj}{W_{i,i'}^{(j,j)}}
\newcommand{\minn}{\min\{\mij,\nj-\mij\}}
\newcommand{\maxx}{\max\{\mij,\nj-\mij\}}

\newcommand\numberthis{\addtocounter{equation}{1}\tag{\theequation}}
\begin{document}
	
\title{\bf A Performance Guarantee for Spectral Clustering}	
\author{March Boedihardjo\thanks{Department of Mathematics, University of California Los Angeles (Email: march@math.ucla.edu).}, Shaofeng Deng\thanks{Department of Mathematics, University of California  Davis (Email: sfdeng@math.ucdavis.edu).},~~and Thomas Strohmer\thanks{Center of Data Science and Artificial Intelligence Research and Department of Mathematics, University of California at Davis (Email: strohmer@math.ucdavis.edu).}\,\,\thanks{M.B.\ acknowledges support from the NSF via grant  DMS 1856221. S.D. and T.S.\ acknowledge support from the NSF via grants DMS 1620455 and DMS 1737943.} }
\maketitle
	
\begin{abstract}
The two-step spectral clustering method, which consists of the Laplacian eigenmap and a rounding step, is a widely used method for graph partitioning. It can be seen as a natural relaxation to the NP-hard minimum ratio cut problem. In this paper we study the central question: when is spectral clustering  able to find the global solution to the minimum ratio cut problem? First we provide a condition that naturally depends on the intra- and inter-cluster connectivities of a given partition under which we may certify that this partition is the solution to the minimum ratio cut problem. Then we develop a deterministic two-to-infinity norm perturbation bound for the the invariant subspace of the graph Laplacian that corresponds to the $k$ smallest eigenvalues. Finally by combining these two results we give a condition under which spectral clustering is guaranteed to output the global solution to the minimum ratio cut problem, which serves as a performance guarantee for spectral clustering. 
\end{abstract}

\section{Introduction}\label{s:intro}
The graph partitioning problem is ubiquitous in data analysis~\cite{BSS20}: how to partition a graph into a given number of subgraphs so that the connections among them are weak? One popular measurement for how well the graph is partitioned is the ratio cut of this partition. Let $G$ be an undirected graph with vertex set $V = \{v_1, \cdots , v_n\}$. We assume that the graph $G$ is weighted, that is each edge between two vertices $v_i$ and $v_j$ carries a non-negative weight $w_{ij} \geq 0$ ($w_{ii}=0$). The weighted adjacency matrix of the graph is the symmetric matrix $W = (w_{ij} )$. Given a $k$-way partition of the vertices $\kpartition$ ($\sqcup_{i=1}^kV_i=V$), the ratio cut of this partition is defined to be
$$\ratiocut{\kpartition}=\sum_{i=1}^{k}\frac{\cut{V_i,V_i^c}}{|V_i|},$$
where $$\cut{V_i,V_i^c}=\sum_{v_j\in V_i, v_k\in V_i^c}w_{jk}$$ is the total weight between $V_i$ and $V_i^c$. The ratio cut measures the connections among the subgraphs normalized by the size of the subgraphs. The purpose of the normalization is to discourage unbalanced partitions. Hence we are interested in finding a $k$-way partition that has the minimum ratio cut, which is presumed to be a NP-hard problem (\cite{goos_between_1993}). Spectral clustering is a natural relaxation to this NP-hard problem. We begin by defining the graph Laplacian of $G$. Let
\[d_i=\deg(v_i)=\sum_{j\neq i}w_{ij}\]
denote the degree of vertex $v_i$. Let the diagonal matrix $D$ be the degree matrix with the degrees $d_1, \cdots , d_n$ on the diagonal. The graph Laplacian of the graph is then defined to be
\[L=D-W.\]
Note that we can rewrite
\begin{align*}
	\text{RatioCut}(\{V_i\}_{i=1}^k)=\sum_{i=1}^{k}\frac{\ones_{V_i}^TL\ones_{V_i}}{|V_i|}=\tr\lp U^TLU\rp,
\end{align*}
where $U\in\mathbb{R}^{n\times k}$ has its $i$th column $U_{\cdot i}$ being $\frac{1}{\sqrt{|V_i|}}\ones_{V_i}$ and $\ones_{V_i}$ is the indicator vector that take value 1 on the vertices in $V_i$ and 0 elsewhere. Therefore the minimum ratio cut problem can be formulated as 
\begin{equation}
	\min_{\kpartition}\tr\lp U^TLU\rp\;\;\;\;\;\text{s.t.} \;\;U_{\cdot i}=\frac{1}{\sqrt{|V_i|}}\ones_{V_i} \, \for i\in[k]. \label{mincut}
\end{equation}
Spectral clustering relaxes the combinatorial constraint of $U$ and instead seeks a solution among all matrices $U$ with orthonormal columns. So the relaxed problem is
\begin{equation}
\min_{U\in\mathbb{R}^{n\times k}}\tr\lp U^TLU\rp\;\;\;\;\;\text{s.t.} \;\;U^TU=I_k, \label{relax}
\end{equation}
whose solution $U$ can be shown to be the eigenvectors w.r.t.\ the $k$ smallest eigenvalues of $L$. Since the columns of $U$ are no longer a collection of indicator vectors, a rounding step is necessary to obtain the partition. The rounding step is performed on the rows of $U$. Namely one should treat the $i$th row $U_{i\cdot}$ as the embedding of vertex $v_i$ in $\mathbb{R}^k$ and obtain the partition by clustering those points (usually through k-means) in $\mathbb{R}^k$. A justification for this idea is the following equivalence form of the relaxed problem~\eqref{relax}:
\begin{equation}
\min_{U\in\mathbb{R}^{n\times k}}\sum_{i=1}^n\sum_{j=1}^nw_{ij}\norm{U_{i\cdot}-U_{j\cdot}}^2\;\;\;\;\;\text{s.t.} \;\;U^TU=I_k. \label{relax2}
\end{equation}
Hence $U_{i\cdot}$ and $U_{j\cdot}$ tend to be close in $\mathbb{R}^k$ if $v_i$ and $v_j$ are strongly connected in $G$. For this reason we call $U$ the Laplacian eigenmap of $G$. Spectral clustering, which consists of the Laplacian eigenmap and a rounding step, is shown in Algorithm~\ref{alg:L}. 

\begin{algorithm}[h!]
	\caption{Spectral clustering}\label{spectralclustering}
	\begin{algorithmic}[1]
		\State {\bf Input:} Weighted adjacency matrix $W$ and the number of clusters $k$.
		\State Compute the graph Laplacian $L = D- W$.
		\State Compute $U\in\mathbb{R}^{n\times k}$ whose columns are the eigenvectors correspond to the $k$ smallest eigenvalues of $L$.
		\State Treat $U_{i\cdot}$ as the embedding of vertex $v_i$ in $\mathbb{R}^k$ and apply clustering method (k-means etc.,) on the points $\lc U_{i\cdot}\rc_{i=1}^n$.
		\State Obtain the partition $\kpartition$ of $V$ based on the result form step 4.
	\end{algorithmic}
	\label{alg:L}
\end{algorithm}

In this paper we try to answer the fundamental question: under what condition is Algorithm~\ref{alg:L}, a relaxation of the minimum ratio cut problem~\eqref{mincut}, able to find the global minimum of~\eqref{mincut}?

\subsection{Related work}
Spectral clustering is a popular graph partition method. We refer the readers to \cite{von_luxburg_tutorial_2007} for an excellent survey on this subject, whose topics include basic properties of the graph Laplacian, variants of spectral clustering methods, constructing similarity graphs from non-graph data, different perspectives of spectral clustering, etc. Even though we have yet to fully understand the mechanism of spectral clustering, some excellent research has been done about its theoretical analysis. One of the most prominent ones is the work on (higher-order) Cheeger-type inequalities~\cite{chung_spectral_1996,lee_multiway_2014}. Another closely related work is~\cite{ling_certifying_2019} which gives performance guarantees for a SDP relaxation to~\eqref{mincut}. In fact our Theorem~\ref{thm:optimality} is a direct improvement to their work. For an analysis of the spectral clustering method on random graphs we refer to~\cite{eldridge_unperturbed_2018,deng_strong_2020,lei_consistency_2015,rohe_spectral_2011,su_strong_2020,abbe_entrywise_2019}.

The technical tool we use is the invariant subspace perturbation theory which studies the change to the invariant subspace of a self-adjoint matrix after the matrix is perturbed. One of the most celebrated works is the classic Davis-Kahan theorem~\cite{davis_rotation_1970} which bounds the invariant subspace perturbation in term of canonical angle. Recent years have witnessed a surge of research on the two-to-infinity norm bound of the invariant subspace perturbation, which is more suitable in many applications. The result we use for this paper is from the remarkable paper by A.~Damle and Y.~Sun~\cite{damle_uniform_2020}. Other related work on this topic includes~\cite{abbe_entrywise_2019,fan_ell_infty_2018,eldridge_unperturbed_2018,cape_two--infinity_2019}.

\subsection{Notation}
We introduce some notation which will be used throughout this paper.
For any matrix $M\in\mathbb{C}^{n\times m}$, we denote by $M_{i\cdot}$ and $M_{\cdot i}$ its $i$th row vector and $i$th column vector respectively.
Moreover, $\norm{M}$ denotes the $\ell^{2}\to \ell^{2}$ induced norm, $\norminf{M} = \max_{i}\normone{M_{i\cdot}}$ denotes the $\ell^{\infty}\to \ell^{\infty}$ induced norm and $\normtwotoinf{M}=\max_{i}\norm{M_{i\cdot}}$ is the $\ell^{2}\to \ell^{\infty}$ induced norm. 
We denote by $\ones_n$ the vector of length $n$ with all entries being 1 and let $J_{n\times m} = \ones_n\ones_m^{\top}$ be the $n\times m$ matrix of all ones. If $S$ is a subset of the vertex set $V$, then $\ones_S$ is the indicator vector such that $(\ones_S)_i=1$ if $v_i\in S$ and $(\ones_S)_i=0$ if $v_i\notin S$. If $M\in\mathbb{C}^{n\times n}$ is self-adjoint, then we arrange its eigenvalues in increasing order:
\[\lambda_1(M)\leq\lambda_{2}(M)\leq\cdots\leq\lambda_n(M).\]

\section{Main results}\label{s:main}
\subsection{Certifying the global minimum of ratio cut}\label{s:main1}
Suppose the partition $\kpartition$ achieves the minimum ratio cut. If we see each $V_i$ as a planted cluster, then the connectivity within each cluster should be strong and the connections between them should be weak. To quantify this, let $L_i\in\mathbb{R}^{|V_i|\times|V_i|}$ be the graph Laplacian of the induced subgraph $G[V_i]$. We measure the connectivity of $G[V_i]$ by $\lambda_2(L_i)$, which is the second smallest eigenvalue of $L_i$. The second smallest eigenvalue of a graph Laplacian is also called the algebraic connectivity of the graph. The larger it is, the stronger the graph is connected. In the case the graph is disconnected, the algebraic connectivity drops to 0. One way to interpret the algebraic connectivity is that it provides a lower bound for the edge density of the graph (see Lemma~\ref{lem:density} below). The proof of this result and subsequent results will be presented in Section~\ref{s:proofs}.
\begin{lemma}~\label{lem:density}
		Let $G$ be a weighted undirected graph with vertex set $V$. Let $L$ be the graph Laplacian of $G$. Let $S$ be a subset of $V$. Then
	$$\cut{S,V-S}\geq\lambda_2(L)\frac{|S|\cdot|V-S|}{|V|}.$$
\end{lemma}
To measure the inter-cluster connectivity, we define for each vertex $v_i$,
$$\ddi=\sum_{v_k\in V_j^c}w_{ik}$$
where $V_j$ is the cluster that contains $v_i$. In other words, $\ddi$ is the total weight between $v_i$ and outside clusters. With such definitions for intra- and inter-cluster connectivity, we are able to certify when a partition is optimal.
\begin{theorem}\label{thm:optimality}
	Suppose a partition $\kpartition$ satisfies
	\begin{equation}\label{cond:sep}
	\max_{1\leq i\leq n}\ddi\leq\frac{1}{2}\min_{1\leq i\leq k}\lambda_2(L_i),
	\end{equation}
	then $\kpartition$ achieves the minimum ratio cut among all $k$-way partitions of $V$. If~\eqref{cond:sep} holds with the strict inequality, then $\kpartition$ is also the unique partition (up to relabeling) that achieves the minimum ratio cut.
\end{theorem}

Theorem~\ref{thm:optimality} is a direct improvement of the result in~\cite{ling_certifying_2019}, which has a constant $\frac{1}{4}$ instead of $\frac{1}{2}$. The following examples show that the constant $\frac{1}{2}$ cannot be further improved.
\begin{example}
	Let $W\in\mathbb{R}^{4n\times 4n}$,
	$$W=\begin{pmatrix}
	J_{n\times n} &J_{n\times n} & cJ_{n\times n}& \\
	J_{n\times n} &J_{n\times n} &              & cJ_{n\times n}\\
	cJ_{n\times n} &              & J_{n\times n}& J_{n\times n}\\
	              &cJ_{n\times n} & J_{n\times n}& J_{n\times n}
	\end{pmatrix}-I_{4n}.$$
	Consider the partition $V_1=\lc v_1,\cdots,v_{2n}\rc$, $V_2=\lc v_{2n+1},\cdots,v_{4n}\rc$. The corresponding $$\min_{1\leq i\leq 2}\lambda_2(L_i)=2n\;\;\;\;\text{and}\;\;\;\; \max_{1\leq i\leq 4n}\ddi=cn.$$
	If $c>1$ then the condition in Theorem~\ref{thm:optimality} is violated for this partition. One can check that in this case a different partition $V^{(1)}=\lc v_1,\cdots,v_{n},v_{2n+1},\cdots,v_{3n}\rc$, $V^{(2)}=V-V^{(1)}$ has a smaller ratio cut.
\end{example}
Theorem~\ref{thm:optimality} is algorithm independent and can be useful in many ways. For example one can use it to check in polynomial time if a given partition is optimal. It can also  serve as a benchmark for comparing different algorithms. In~\cite{ling_certifying_2019} the authors propose a SDP relaxation to the minimum ratio cut problem~\eqref{mincut} and show that it is able to find the optimal partition if it satisfies $\max_{1\leq i\leq n}\ddi<\frac{1}{4}\min_{1\leq i\leq k}\lambda_2(L_i)$. In this paper we prove  that Algorithm~\ref{spectralclustering} is able to find the optimal partition if $\max_{1\leq i\leq n}\ddi\lesssim\frac{1}{\ln n}\min_{1\leq i\leq k}\lambda_2(L_i)$. The notation ``$\lesssim$'' hides a term that does not depend on $n$.

\subsection{A two-to-infinity norm bound for the Laplacian eigenmap}\label{s:main2}
Algorithm~\ref{spectralclustering} can be understood from a perturbation perspective. Suppose we try to recover the planted partition $\kpartition$. Let $W_i$, $D_i$, $L_i$ denote the weighted adjacency matrix, degree matrix and graph Laplacian of the induced subgraph $G[V_i]$ respectively. Let
$$\Wiso=\begin{pmatrix}
W_1&&&\\
&W_2&&\\
&&\ddots&\\
&&&W_k
\end{pmatrix}, \Wd=W-\Wiso.$$
Let $\Diso$, $\Liso$, $\Dd$, $\Ld$ be the corresponding degree matrices or graph Laplacians (here we suppose $\lambda_{k+1}(\Liso)>0$). Let $U$ ($\Uiso$) be a matrix with orthonormal columns whose range is the invariant subspace of $L$ ($\Liso$) that corresponds to the $k$ smallest eigenvalues. Then the $k$ smallest eigenvalues of $\Liso$ are 0 and $\Uiso$, up to a multiplication of orthogonal matrix from the right, is
$$\Uiso=\lp\frac{1}{\sqrt{|V_1|}}\ones_{V_1}\;\;\;\frac{1}{\sqrt{|V_2|}}\ones_{V_2}\;\;\;\cdots\;\;\;\frac{1}{\sqrt{|V_k|}}\ones_{V_k}\rp.$$
Hence the rows of $\Uiso$ reduce to $k$ different points in $\mathbb{R}^k$ with one cluster at each point. Any rounding method will recover the planted clusters perfectly. Here we also point out that a multiplication of orthogonal matrix from the right transforms all the rows simultaneously and thus preserves the geometry of the embedding. If $\Wd$ is small, then $U$ should be close to $\Uiso$. A reasonable measurement for the closeness is 
$$\min_{V\in{\bf O}^k}\normtwotoinf{UV-\Uiso}=\min_{V\in{\bf O}^k}\max_{1\leq i\leq n}\norm{\lp UV-\Uiso\rp_{i\cdot}},$$
where the minimization is taken over all $k\times k$ orthogonal matrices. This error measures the maximum distance of a point $U_{i\cdot}$ away from its origin $(\Uiso)_{i\cdot}$ after some global orthogonal transformation. If this error is small enough then the rounding step should be able to recover the planted clusters perfectly. We present our bound for $\min_{V\in{\bf O}^k}\normtwotoinf{UV-\Uiso}$ in Theorem~\ref{thm:twoinfbound} below. The result is stated in terms of $||U\tilde{V}-\Uiso||_{2,\infty}$ where $\tilde{V}$ solves the orthogonal Procrustes problem
$$\tilde{V}=\arg\min_{V\in{\bf O}^k} \normfro{UV-\Uiso}.$$
Note that $\min_{V\in{\bf O}^k}\normtwotoinf{UV-\Uiso}\leq||U\tilde{V}-\Uiso||_{2,\infty}$.
\begin{theorem}\label{thm:twoinfbound}
	Suppose each $|V_i|\geq 3$. Let 
	$$c=\max_{1\leq i\leq k}\frac{n}{|V_i|}\;\;\;\text{and}\;\;\;r = \frac{\max_{1\leq i\leq n}\ddi}{\min_{1\leq i\leq k}\lambda_2(L_i)}$$ be the unbalanceness and the perturbation/eigengap ratio respectively. If $r\leq\frac{1}{16(1+c)\ln n}$, then
	$$\normtwotoinf{U\tilde{V}-\Uiso}\leq32\sqrt{c}\lp r^2+r\ln n\rp\frac{1}{\sqrt{n}}.$$
\end{theorem}
The rest of this section is dedicated to the technical details of the proof. Discussions and applications regarding this bound are deferred to Section~\ref{s:discussion}. The tool we use is  Corollary~3.3 in\cite{damle_uniform_2020} which gives a two-to-infinity norm perturbation bound for the invariant subspace. We cite this result in Lemma~\ref{lem:twoinfbound} below. The definition of  the separation of two matrices (denoted by {\em sep}) that arises in Lemma~\ref{lem:twoinfbound}, is stated below the lemma.
\begin{lemma}\label{lem:twoinfbound}
	Let $\Liso=\Uiso\Lambda_1\Uiso^T+U_2\Lambda_2U_2^T$ be the spectral decomposition of $\Liso$ where $\Lambda_1\in\mathbb{R}^{k\times k}$ is a zero matrix and $\Lambda_2\in\mathbb{R}^{(n-k)\times (n-k)}$ whose diagonal contains all the positive eigenvalues of $\Liso$. Let $\gap=\min\lc\sep_2(\Lambda_1,\Lambda_2),\sep_{(\twotoinf),U_2}(\Lambda_1,U_2\Lambda_2U^T_2)\rc$ and $\mu=\sqrt{n}\normtwotoinf{\Uiso}$. If $\norm{\Ld}\leq\frac{\gap}{5}$ and $\norminf{\Ld}\leq\gap/(4+4\mu^2)$ then
	$$\normtwotoinf{U\tilde{V}-\Uiso}\leq 8\normtwotoinf{\Uiso}\lp\frac{\norm{\Ld}}{\sep_2(\Lambda_1,\Lambda_2)}\rp^2+4\frac{\normtwotoinf{U_2U_2^T\Ld\Uiso}}{\gap}.$$
\end{lemma}
Classical perturbation theory like Davis-Kahan usually bounds the invariant subspace perturbation in terms of the perturbation/eigengap ratio. Lemma~\ref{lem:twoinfbound} is similar but with the classical eigengap replaced by the {\em gap} term defined therein. Here the separation of two matrices is defined as 
$$\sep_{*,W}(B,C)=\inf\lc||ZB-CZ||_*:Z\in\mathbb{R}^{m\times l}, \ran Z\subseteq\ran W,||Z||_*=1\rc$$
where $B\in\mathbb{R}^{l\times l}$, $C\in\mathbb{R}^{m\times m}$, $\ran W$ is a linear subspace of $\mathbb{R}^m$ and $||\cdot||_*$ is a norm on $\mathbb{R}^{m\times l}$. When $\ran W=\mathbb{R}^m$ we denote $\sep_{*}(B,C)=\sep_{*,W}(B,C)$. $\norm{\Ld}$ and $\norminf{\Ld}$ in Lemma~\ref{lem:twoinfbound} can be bounded by $2\max_{1\leq i\leq n}\ddi$. And the two matrix separation terms $\sep_2(\Lambda_1,\Lambda_2)$ and $\sep_{(\twotoinf),U_2}(\Lambda_1,U_2\Lambda_2U^T_2)$ are closely related to the eigengap $\min_{1\leq i\leq k}\lambda_2(L_i)$. In fact 
\begin{align*}
	\sep_2(\Lambda_1,\Lambda_2)&=\inf\lc||Z\Lambda_1-\Lambda_2Z||_2:Z\in\mathbb{R}^{(n-k)\times k},||Z||_2=1\rc\\
	&=\inf\lc||\Lambda_2Z||_2:Z\in\mathbb{R}^{(n-k)\times k},||Z||_2=1\rc\\
	&=\min_{1\leq i\leq k}\lambda_2(L_i) \numberthis\label{eq:sepdiag}
\end{align*}
is exactly the eigengap. And
\begin{align*}
\sep_{(\twotoinf),U_2}(\Lambda_1,U_2\Lambda_2U^T_2)&=\inf\lc\normtwotoinf{0-\Liso Z}:Z\in\mathbb{R}^{n\times k},\ran Z\subseteq\ran U_2,\normtwotoinf{Z}=1\rc\\
&=\min_{1\leq i\leq k}\inf\lc\normtwotoinf{L_i Z}:Z\in\mathbb{R}^{|V_i|\times k},\ran Z\subseteq\lc\ones_{|V_i|}\rc^\perp,\normtwotoinf{Z}=1\rc\\
&=\min_{1\leq i\leq k}\inf_{x\perp \ones_{|V_i|}}\frac{\norminf{L_ix}}{\norminf{x}}
\end{align*}
can be understood as the ``eigengap'' in terms of the $\ell^{\infty}$ norm. The third equality is because for any $Z$ if we let $x\neq 0$ be the vector that $\norminf{Zx}=\norm{x}$, then
$$\normtwotoinf{L_iZ}\geq\frac{\norminf{L_iZx}}{\norm{x}}=\frac{\norminf{L_iZx}}{\norminf{Zx}}\geq\inf_{x\perp \ones_{|V_i|}}\frac{\norminf{L_ix}}{\norminf{x}}.$$
And on the other hand we can pick a $\tilde{Z}$ such that its first column satisfies $$\norminf{\tilde{Z}_{\cdot1}}=1\;,\;\;\;\;\norminf{L_i\tilde{Z}_{\cdot1}}=\inf_{x\perp \ones_{|V_i|}}\frac{\norminf{L_ix}}{\norminf{x}}$$
and its other columns are 0 so that $$\normtwotoinf{L_i\tilde{Z}}=\inf_{x\perp \ones_{|V_i|}}\frac{\norminf{L_ix}}{\norminf{x}}.$$
Note that we always have
$$\inf_{x\perp \ones_{|V_i|}}\frac{\norminf{L_ix}}{\norminf{x}}\leq\lambda_2(L_i).$$
Therefore the the {\em gap} term in Lemma~\ref{lem:twoinfbound} is simplified to
\begin{equation}
	\gap=\min_{1\leq i\leq k}\inf_{x\perp \ones_{|V_i|}}\frac{\norminf{L_ix}}{\norminf{x}}.\label{eq:gap}
\end{equation}
There is a trivial bound that relates {\em gap} to the eigengap:
$$\inf_{x\perp \ones_{|V_i|}}\frac{\norminf{L_ix}}{\norminf{x}}\geq\frac{\lambda_2(L_i)}{\sqrt{|V_i|}}.$$
But we will show that due to the diagonally dominant structure of the graph Laplacian, the $\sqrt{|V_i|}$ factor can be improved to $\ln |V_i|$. The following theorem, which is also of independent interest, is essential in this context.
\begin{theorem}\label{lowerbound}
	Let $B$ be a self-adjoint $n\times n$ matrix, $n\geq 3$ such that $B_{i,i}\geq\sum_{j\in\{1,\ldots,n\}\backslash\{i\}}|B_{i,j}|$ for all $1\leq i\leq n$. Let $\mathcal{M}$ be a subspace of $\mathbb{C}^{n}$ such that $B\mathcal{M}\subset\mathcal{M}$. Then
	\[\|Bx\|_{\infty}\geq\frac{\lambda_{\min}(B|_{\mathcal{M}})\|x\|_{\infty}}{2\ln n},\]
	for all $x\in\mathcal{M}$.
\end{theorem}
\begin{corollary}\label{graph}
	Suppose that $L$ is the Laplacian of a graph with $n$ vertices, $n\geq 3$. Then
	\[\frac{\lambda_{2}(L)}{2\ln n}\leq\inf_{x\perp\ones_n}\frac{\|Lx\|_{\infty}}{\|x\|_{\infty}}\leq\frac{4M}{D},\]
	where the second inequality holds for unweighted graphs with $M$ being the maximum degree and $D$ being the diameter.
\end{corollary}
The following example shows that the $\ln n$ factor in Corollary~\ref{graph} is necessary. However, at this point we do not know whether it must carry over to $||U\tilde{V}-\Uiso||_{2,\infty}$ as well.
 \begin{example}
 	Suppose that $L$ is the Laplacian of a $d$-regular Ramanujan graph with $n$ vertices. This means that $\lambda_{2}(L)\geq d-2\sqrt{d-1}$. Note that $n\leq(d+1)^{D}$ where $D$ is the diameter of the graph. This follows from the fact that for a fixed vertex $u_{0}$, every vertex can be connected to $u_{0}$ via a path of length at most $D$ and that there are at most $1+d+d^{2}+\ldots+d^{D}\leq(d+1)^{D}$ paths of length at most $D$. Thus, $D\geq\frac{\ln n}{\ln(d+1)}$. By Corollary \ref{graph},
 	\[\frac{d-2\sqrt{d-1}}{2\ln n}\leq\inf_{x\perp\ones_n}\frac{\|Lx\|_{\infty}}{\|x\|_{\infty}}\leq\frac{4d}{D}\leq\frac{4d\ln(d+1)}{\ln n}.\]
 	For example, if $L$ is the Laplacian of a $5$-regular Ramanujan graph with $n$ vertices, then
 	\[\frac{1}{2\ln n}\leq\inf_{x\perp\ones_n}\frac{\|Lx\|_{\infty}}{\|x\|_{\infty}}\leq\frac{36}{\ln n}.\]
 	For every $d\geq 3$, there exist infinitely many $d$-regular Ramanujan graphs by~\cite{marcus2013interlacing}. This shows that the $\ln n$ factor in Corollary \ref{graph} is necessary.
 \end{example}

\section{Discussions}\label{s:discussion}
Algorithm~\ref{spectralclustering} consists of two steps: the Laplacian eigenmap and a rounding step. We have bounded the Laplacian eigenmap in Theorem~\ref{thm:twoinfbound}. The next question is whether the rounding step will successfully recover the planted clusters based on the embedded points in $\mathbb{R}^k$. The answer depends on our understanding of the choice of the rounding method and it is beyond the scope of this paper to present a survey on this subject. But we will show through several examples that the condition
\begin{equation}
	\normtwotoinf{U\tilde{V}-\Uiso}< \frac{C}{\sqrt{n}} \label{eq:sep}
\end{equation}
should imply successful recovery, where $C$ depends on the specific choice of the rounding method and (possibly) the number of clusters $k$.
 \begin{itemize}[leftmargin=*]
 	\item {\bf A simple bisector for two clusters.} When $k=2$, the Fiedler eigenvector (i.e., the eigenvector $u_2(L)$ that corresponds to the second smallest eigenvalue of $L$) is a popular tool to partition the graph. One way to do this is to first put the entries of the Fiedler eigenvector in algebraic order. Then out of all $n-1$ possible linear bisections of the entries we pick the one that gives the smallest ratio cut. This method is equivalent to finding the best linear bisection of the embedded points in $\mathbb{R}^2$. For this rounding method we can let $C=1$ in~\eqref{eq:sep}. To see why, first note that $\tilde{V}$ is the solution to an orthogonal Procrustes problem and therefore has a closed form solution
 	$$\tilde{V}=V_1V_2^T$$
 	where $U^T\Uiso=V_1\Sigma V_2^T$ is the singular value decomposition of $U^T\Uiso$. Given that $u_1(L)=u_1(\Liso)=\frac{1}{\sqrt{n}}\ones_n$, it is easy to check that
 	$$\normtwotoinf{U\tilde{V}-\Uiso}=\norminf{u_2(L)-u_2(\Liso)}$$
 	where the sign of $u_2(L)$ is chosen so that $\inner{u_2(\Liso)}{u_2(L)}>0$. Note that the distance between the two embedded unperturbed clusters is $\sqrt{1/|V_1|+1/|V_2|}$. To ensure the separation of the two clusters in $u_2(L)$ we require
 	$$\norminf{u_2(L)-u_2(\Liso)}<\frac{1}{2}\sqrt{\frac{1}{|V_1|}+\frac{1}{|V_2|}},$$
 	which is guaranteed by $\normtwotoinf{U\tilde{V}-\Uiso}< 1/\sqrt{n}$.
 	
 	\item {\bf An SDP type of k-means algorithm.} The k-means algorithms are a family of algorithms that seek the $k$-way partition $\lc \Gamma_i\rc_{i=1}^k$ of $n$ points in $\mathbb{R}^d$ by minimizing the following k-means objective function:
 	\[\min_{\lc \Gamma_i\rc_{i=1}^k}\sum_{i=1}^k\sum_{x\in\Gamma_i}\norm{x-\mu_i}^2\]
 	where $\mu_i$ is the mean of points in $\Gamma_i$. Note that the optimization is shown to be NP-hard (\cite{mahajan_planar_2009,aloise_np-hardness_2009}) so no polynomial-time algorithm is guaranteed to find the optimal partition in general. The most famous and widely used k-means algorithm is the Lloyd's algorithm (\cite{lloyd_least_1982}). But its heuristic nature and random initial starts make the analysis of exact recovery difficult. Here we consider a SDP type of k-means algorithm proposed in~\cite{li_when_2020}. The proposed algorithm comes with a proximity condition for the planted partition under which the algorithm is guaranteed to recover the planted partition. Let $n_i=|\Gamma_i|$ and $X_i\in\mathbb{R}^{n_i\times d}$ be the data matrix of the $i$-th cluster with each row being a point in $\Gamma_i$. Let $\overline{X}_i=X_i-\ones_{n_i}\mu_i^T$ be the centered data matrix of the $i$-th cluster. For each pair of $i\neq j$, let $M_{i,j}$ denote the bisecting hyperplane that passes through $\frac{\mu_i+\mu_j}{2}$ and is perpendicular to the line segment that joins $\mu_i$ and $\mu_j$. The proximity condition is then stated as follows: for all $i\neq j$, (i) $\Gamma_i$ and $\Gamma_j$ are separated by $M_{i,j}$ and (ii) 
 	$$\xi_{i,j}>\frac{1}{2}\sqrt{\sum_{i=1}^k\norm{\overline{X}_i}^2\lp\frac{1}{n_i}+\frac{1}{n_j}\rp}$$
 	where $\xi_{i,j}=\dist\lc M_{i,j},\Gamma_i\cup\Gamma_j\rc$ is the margin between the clusters and the bisecting hyperplane (see also Figure~\ref{fig:prox}). We now claim that if $C=1/5$ in~\eqref{eq:sep} then the SDP k-means algorithm is guaranteed to recover the planted partition. To see why, note that~\eqref{eq:sep} implies the points in $\Gamma_i$ and $\Gamma_j$, together with their means, are confined within two balls of radius $C/\sqrt{n}$ whose centers are $\sqrt{\frac{1}{n_i}+\frac{1}{n_j}}$ apart. By Lemma~\ref{lem:sep} in Section~\ref{s:lemma} we have
 	$$\xi_{i,j}>\frac{1}{2}\sqrt{\frac{1}{n_i}+\frac{1}{n_j}}-\frac{3C}{\sqrt{n}}\geq\frac{1-3C}{2}\sqrt{\frac{1}{n_i}+\frac{1}{n_j}}.$$
 	And on the other hand 
 	$$\sum_{i=1}^k\norm{\overline{X}_i}^2\leq\sum_{i=1}^k\normfro{\overline{X}_i}^2< 4C^2.$$
 	Therefore if $C\leq 1/5$ then the proximity condition is satisfied.
	\item[]
	\begin{minipage}{\linewidth}
	\centering
	\includegraphics[width=130mm]{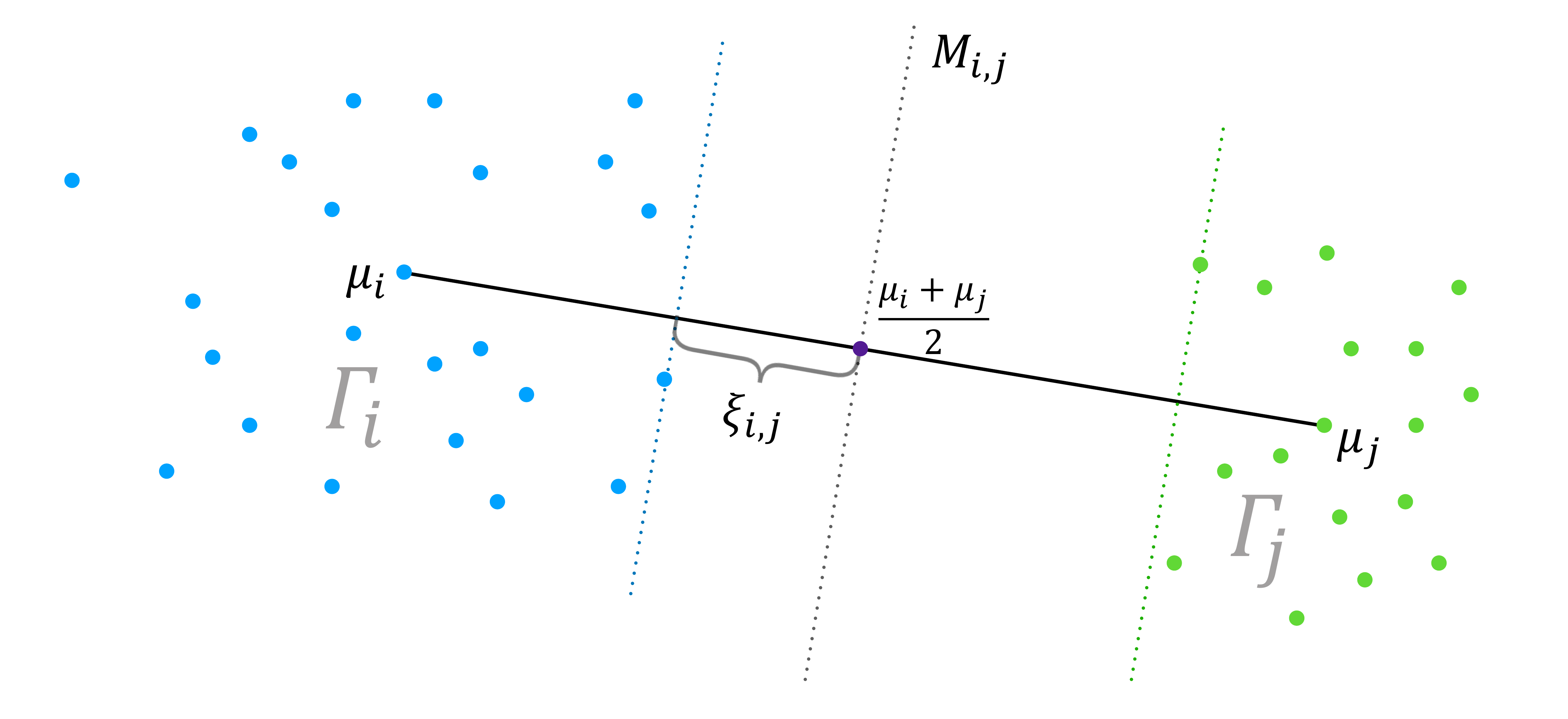}
	\captionof{figure}{Proximity condition for the SDP k-means algorithm. If the partition satisfies the proximity condition, then each pair of clusters $\Gamma_i$ and $\Gamma_j$ are separated by and sufficiently bounded away from the bisecting hyperplane of the line segment that joins $\mu_i$ and $\mu_j$.
	}\label{fig:prox}
	\end{minipage}

 	\item {\bf Two projective k-means algorithms.} In~\cite{kumar_clustering_2010} and~\cite{awasthi_improved_2012} two projective k-means algorithms are proposed which consist of an SVD-based projection followed by iterative Lloyd steps with informed initial starts. By our notation, the algorithm in~\cite{kumar_clustering_2010} is guaranteed to recover the planted partition if for any $i\neq j$,
 	$$\xi_{i,j}>\tilde{C}k\lp\frac{1}{\sqrt{n_i}}+\frac{1}{\sqrt{n_j}}\rp\norm{\overline{W}}$$
 	where $\tilde{C}>0$ is an absolute constant and $\overline{W}=[\overline{X}_1^T,\cdots,\overline{X}_k^T]^T$. The algorithm in~\cite{awasthi_improved_2012} is guaranteed to recover the planted partition if for any $i\neq j$,
 	$$\xi_{i,j}>\frac{1}{2}\tilde{C}\lp\frac{1}{\sqrt{n_i}}+\frac{1}{\sqrt{n_j}}\rp\norm{\overline{W}}$$
 	and
 	$$\norm{\mu_i-\mu_j}>\tilde{C}\sqrt{k}\lp\frac{1}{\sqrt{n_i}}+\frac{1}{\sqrt{n_j}}\rp\norm{\overline{W}}.$$
 	Then $C$ in~\eqref{eq:sep} can be similarly derived for both methods.
 \end{itemize}
Thus, by Theorem~\ref{thm:twoinfbound} and the discussion above, Algorithm~\ref{spectralclustering} finds the planted partition if $r\lesssim 1/\ln n$ where the hidden term depends on $k$ and the unbalanceness term $c$. By Theorem~\ref{thm:optimality} when $r\leq 1/2$, we can certify that the planted partition is optimal. Therefore we may claim that Algorithm~\ref{spectralclustering} finds the optimal partition when $r\lesssim 1/\ln n$.

Note that when the unbalanceness term $c$ gets arbitrarily large, Theorem~\ref{thm:twoinfbound} gets arbitrarily bad. This is to be expected. We illustrate the effect of unbalanceness on the Laplacian eigenmap using a simple numerical example. Consider $|V_1| = 3$ and $|V_2|=|V_3|=300$. Let $W_i=J_{|V_i|\times|V_i|}/|V_i|$ so $\lambda_{2}(L_i)=1$ for $1\leq i\leq 3$. We perturb the graph by adding a weight 0.5 between a vertex in $V_1$ and a vertex in $V_2$. We then add another weight 0.5 between a vertex in $V_3$ and another vertex in $V_2$. The Laplacian eigenmap of both the unperturbed and the perturbed graphs are shown in Figure~\ref{fig}. As seen from the figure, $ \min_{V\in{\bf O}^k}\normtwotoinf{UV-\Uiso}$ is large and will be even more so as the clusters get more unbalanced. Despite the error being large, the Laplacian eigenmap still separates the three clusters well. The reason being that $ \min_{V\in{\bf O}^k}\normtwotoinf{UV-\Uiso}$ only measures the magnitude and thus ignores the direction of the perturbation. A more refined analysis on the Laplacian eigenmap should consider the direction of the perturbation as well as the magnitude. It remains an open problem whether there exists a constant $C$ such that $r\leq C$ implies successful recovery of the planted clusters by Algorithm~\ref{spectralclustering}. The constant $C$ should only depend on $k$ or, better yet,  not even on $k$.   
\begin{figure}[h!]
	
	\begin{subfigure}[h]{0.5\linewidth}
		\includegraphics[width=\linewidth]{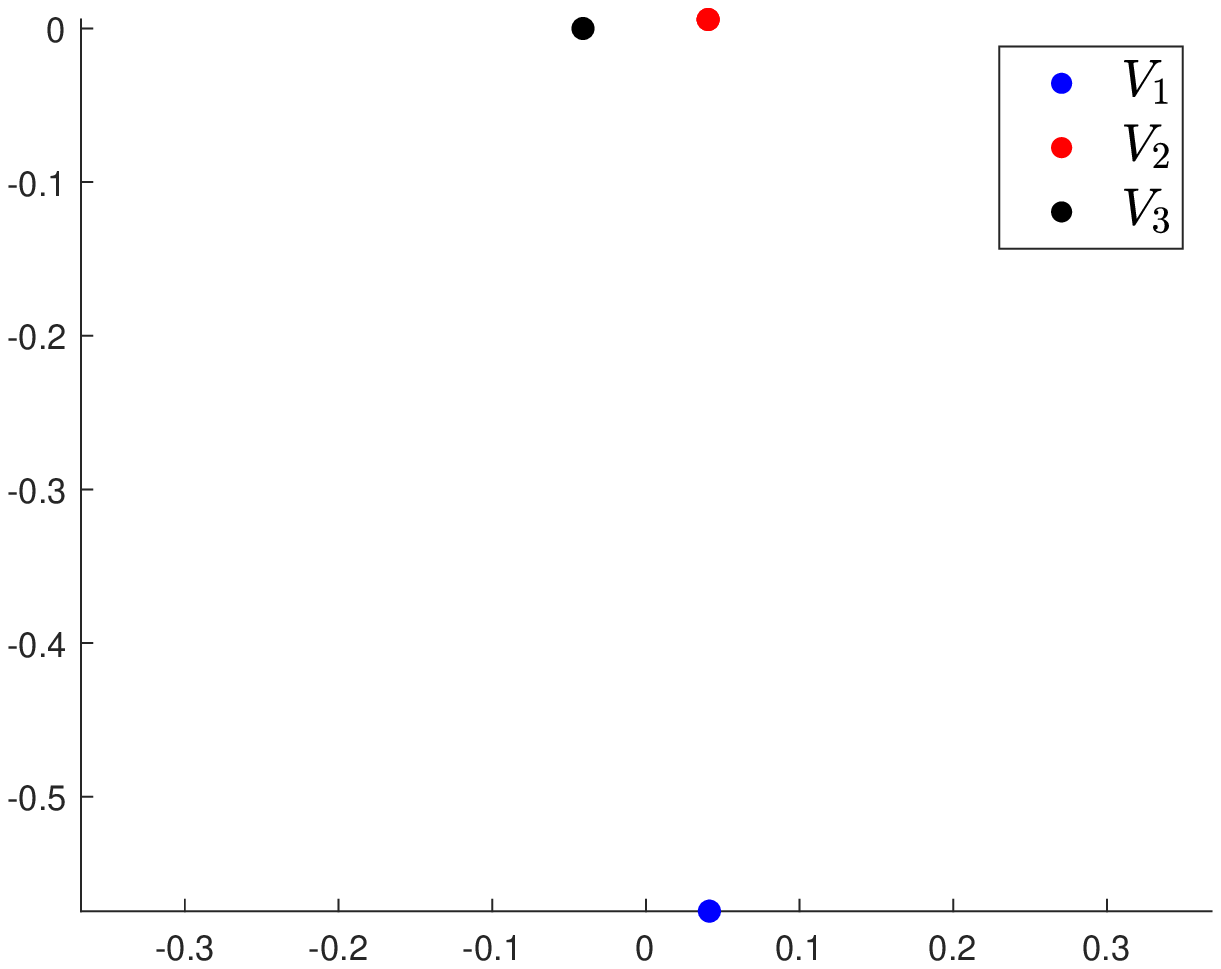}
		\caption{Unperturbed graph}
	\end{subfigure}
	\hfill
	\begin{subfigure}[h]{0.5\linewidth}
		\includegraphics[width=\linewidth]{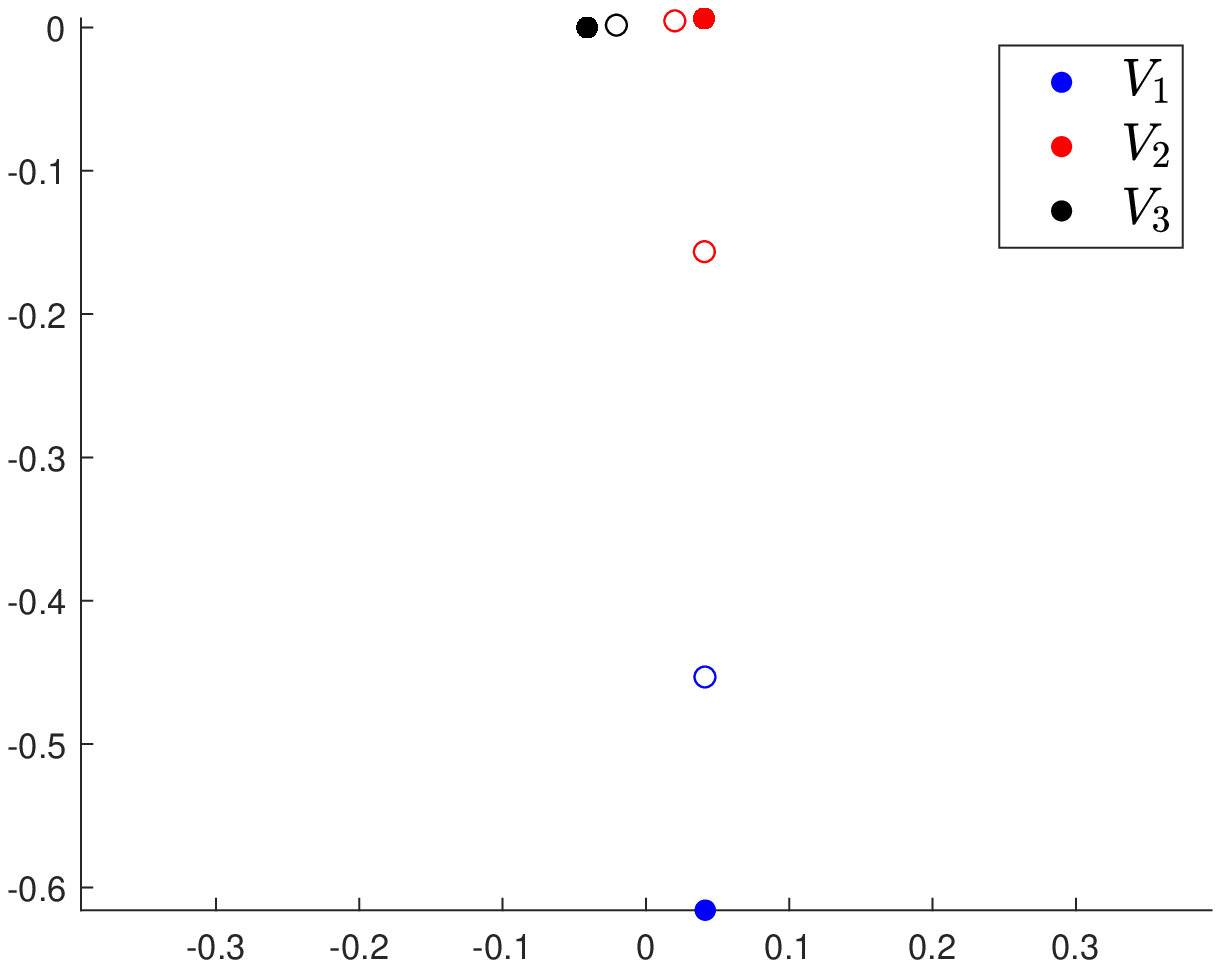}
		\caption{Perturbed graph}
	\end{subfigure}%
	
	\caption{Laplacian eigenmap for both the unperturbed and the perturbed graphs. The embedded points live on a common plane in $\mathbb{R}^3$ so we can visualize them in $\mathbb{R}^2$. The four empty circles are the four perturbed vertices.}
	\label{fig}
\end{figure}

\section{Proofs}\label{s:proofs}

\subsection{Proofs for Section~\ref{s:main1}}
\begin{proof}[\bf Proof of Lemma~\ref{lem:density}]
	Consider $M=L-\lambda_2(L)P$ where $P=I-\frac{1}{|V|}J_{|V|\times |V|}$. Then $M$ is positive semi-definite. Therefore
	$$\ones_S^TM\ones_S=\ones_S^TL\ones_S-\lambda_2(L)\ones_S^TP\ones_S=\cut{S,V-S}-\lambda_2(L)\frac{|S|\cdot|V-S|}{|V|}\geq 0$$
\end{proof}
\begin{proof}[\bf Proof of Theorem~\ref{thm:optimality}]
	Suppose without loss of generality that the partition $\kpartition$ satisfies 
	$$\max_{1\leq i\leq n}\ddi\leq\frac{1}{2}\;\;\;\text{and}\;\;\;\min_{1\leq i\leq k}\lambda_2(L_i)\geq 1.$$
	Let $\kpartitionnew$ be another partition of $V$. We aim to show
	$$\ratiocut{\kpartition}\leq\ratiocut{\kpartitionnew}.$$
	Let $\ni=|V_i|$, $\nj=|V^{(j)}|$ and $\mij=|V_i\cap V^{(j)}|$. We have
	$$\sumi\mij=\nj,\;\;\;\sumj\mij=\ni.$$
	Let
	$$\Aiipjjp=\sum_{v_k\in V_i\cap V^{(j)},v_l\in V_{i'}\cap V^{(j')}}w_{kl}.$$
	Thus we have divided the weighted adjacency matrix $W$ into $k^2\times k^2$ rectangular areas and $\Aiipjjp$ is the total weight in one of the areas. Lemma~\ref{lem:density} gives 
	\begin{equation}
	\sumjpneqj\Aiijjp\geq\min_{1\leq a\leq k}\lambda_2(L_a)\frac{\mij(\ni-\mij)}{\ni}\geq \frac{\mij(\ni-\mij)}{\ni} \label{1}
	\end{equation}
	for all $i,j$. We also know each $\ddi$ is at most $1/2$. This implies 
	\begin{equation}
	\sumipneqi\Aiipjj+\sumjpneqj\sumipneqi\Aiipjjp=\sumjp\sumipneqi\Aiipjjp\leq\frac{1}{2}\mij \label{2}
	\end{equation}
	for all $i,j$. Moreover 
	\begin{equation}
	\sumipneqi\Aiipjj\leq \frac{1}{2}\minn \label{3}
	\end{equation}
	for all $i,j$ because the summands together represent a rectangular area in $W$ with length $\mij$ and width $\nj-\mij$. Therefore we need to show
	\begin{align*}
	&\ratiocut{\kpartition}-\ratiocut{\kpartitionnew}\\
	=&\sumj\sumjp\sumi\sumipneqi\frac{1}{\ni}\Aiipjjp-\sumi\sumip\sumj\sumjpneqj\frac{1}{\nj}\Aiipjjp\\
	=&\sumi\sumj\sumipneqi\frac{1}{\ni}\Aiipjj+\sumi\sumj\left( \frac{1}{\ni}-\frac{1}{\nj}\right)\lp\sumjpneqj\sumipneqi\Aiipjjp\rp-\sumi\sumj\sumjpneqj\frac{1}{\nj}\Aiijjp\\
	=&[1]+[2]-[3]\leq 0.
	\end{align*}
	For [1],
	\begin{align*}
	[1]=&\sumi\sumj\sumipneqi\frac{1}{\ni}\Aiipjj\\
	=&\sumi\sumj\sumipneqi\frac{1}{\ni}\Aiipjj-\frac{1}{2}\sumi\sumj\frac{1}{\ni}\minn\\
	&+\frac{1}{2}\sumi\sumj\frac{1}{\ni}\minn\\
	=&[1.1]-[1.2]+[1.3].
	\end{align*}
	For [2], when $\frac{1}{\ni}-\frac{1}{\nj}\leq 0$, the summand is upper bounded by 0. When $\frac{1}{\ni}-\frac{1}{\nj}> 0$, the summand is upper bounded by (by using~\eqref{2}) 
	$$\lp\frac{1}{\ni}-\frac{1}{\nj}\rp\lp\sumjpneqj\sumipneqi\Aiipjjp\rp\leq\lp\frac{1}{\ni}-\frac{1}{\nj}\rp\lp\frac{1}{2}\mij-\sumipneqi\Aiipjj\rp.$$
	Therefore we can bound [2] by
	\begin{align*}
	[2]&=\sumi\sumj\left( \frac{1}{\ni}-\frac{1}{\nj}\right)\lp\sumjpneqj\sumipneqi\Aiipjjp\rp\\
	&\leq \sumi\sumj\max\lc\frac{1}{\ni}-\frac{1}{\nj},0\rc\lp\frac{1}{2}\mij-\sumipneqi\Aiipjj\rp\\
	&=\sumi\sumj\lp\frac{1}{\ni}-\min\lc\frac{1}{\nj},\frac{1}{\ni}\rc\rp\lp\frac{1}{2}\mij-\sumipneqi\Aiipjj\rp\\
	&\leq\frac{1}{2}\sumi\sumj\frac{1}{\ni}\mij-\sumi\sumj\sumipneqi\frac{1}{\ni}\Aiipjj-\frac{1}{2}\sumi\sumj\min\lc\frac{1}{\nj},\frac{1}{\ni}\rc\max\lc 0,2\mij-\nj\rc\\
	&=[2.1]-[2.2]-[2.3],
	\end{align*}
	where in the last step we used~\eqref{3}. For [3], we use~\eqref{1}:
	\begin{align*}
	[3]=\sumi\sumj\sumjpneqj\frac{1}{\nj}\Aiijjp&\geq\sumi\sumj\min_{1\leq a\leq k}\lambda_2(L_a)\frac{\mij(\ni-\mij)}{\ni\nj}\\
	&\geq \sumi\sumj\frac{\mij(\ni-\mij)}{\ni\nj}.\numberthis\label{eq:4.4}
	\end{align*}
	Therefore (here we introduce the shorthand notation $\min^*$ for $\minn$ and $\max^*$ for $\maxx$)
	\begin{align*}
	[1.3]-[3]&\leq\sumi\sumj\lp\frac{1}{2\ni}\minn-\frac{\mij(\ni-\mij)}{\ni\nj}\rp\\
	&=\sumi\sumj\lp\frac{\min^*(\min^*+\max^*)}{2\ni\nj}-\frac{\mij(\ni-\mij)}{\ni\nj}\rp\\
	&=\sumi\sumj\lp\frac{\min^*\max^*}{\ni\nj}-\frac{\min^*(\max^*-\min^*)}{2\ni\nj}-\frac{\mij(\ni-\mij)}{\ni\nj}\rp\\
	&=\sumi\sumj\lp\frac{\mij(\nj-\mij)}{\ni\nj}-\frac{\mij(\ni-\mij)}{\ni\nj}\rp-\frac{1}{2}\sumi\sumj\frac{\min^*(\max^*-\min^*)}{\ni\nj}\\
	&=\sumi\sumj\lp\frac{\mij(\nj-\ni)}{\ni\nj}\rp-[4]\\
	&=\sumi\sumj\frac{\mij}{\ni}-\sumj\sumi\frac{\mij}{\nj}-[4]\\
	&=k-k-[4]=-[4].
	\end{align*}
	The cancellation above is the reason why the constant in the condition~\eqref{cond:sep} is 1/2. We also have [1.1]-[2.2]=0. Finally
	\begin{align*}
	&[1]+[2]-[3]\\
	\leq&[2.1]-[1.2]-[4]-[2.3]\\
	=&\frac{1}{2}\sumi\sumj\lp\frac{\mij}{\ni}-\frac{\min^*}{\ni}-\frac{\min^*(\max^*-\min^*)}{\ni\nj}-\min\lc\frac{1}{\nj},\frac{1}{\ni}\rc\max\lc 0,2\mij-\nj\rc\rp\\
	=&\frac{1}{2}\sumi\sumj\lp\frac{\mij}{\ni}-\frac{\min^*(\min^*+\max^*)}{\ni\nj}-\frac{\min^*(\max^*-\min^*)}{\ni\nj}-\min\lc\frac{1}{\nj},\frac{1}{\ni}\rc\max\lc 0,2\mij-\nj\rc\rp\\
	=&\frac{1}{2}\sumi\sumj\lp\frac{\mij}{\ni}-\frac{2\min^*\max^*}{\ni\nj}-\min\lc\frac{1}{\nj},\frac{1}{\ni}\rc\max\lc 0,2\mij-\nj\rc\rp\\
	=&\frac{1}{2}\sumi\sumj\lp\frac{\mij\nj}{\ni\nj}-\frac{2\mij(\nj-\mij)}{\ni\nj}-\min\lc\frac{1}{\nj},\frac{1}{\ni}\rc\max\lc 0,2\mij-\nj\rc\rp\\
	=&\frac{1}{2}\sumi\sumj\lp\frac{\mij(2\mij-\nj)}{\ni\nj}-\min\lc\frac{1}{\nj},\frac{1}{\ni}\rc\max\lc 0,2\mij-\nj\rc\rp\\
	\leq& 0.
	\end{align*}
	The last step is because for each summand,
	$$\frac{\mij}{\ni\nj}\leq\min\lc\frac{1}{\nj},\frac{1}{\ni}\rc.$$
	This concludes the proof that $\kpartition$ achieves the minimum ratio cut. To show that it is also the unique global minimum when the strict inequality holds, we assume $\kpartition$ satisfies 
	$$\max_{1\leq i\leq n}\ddi\leq\frac{1}{2}\;\;\;\text{and}\;\;\;\min_{1\leq i\leq k}\lambda_2(L_i)> 1.$$
	All claims above still hold but we will show~\eqref{eq:4.4} holds with the strict inequality. Since $\min_{1\leq i\leq k}\lambda_2(L_i)> 1$, the last inequality in~\eqref{eq:4.4} takes equal sign if and only if $\mij(\ni-\mij)=0$ for all $i,j$. But this is impossible if $\kpartitionnew$ is not a relabeling of $\kpartition$. Therefore if $\kpartitionnew$ is not a relabeling of $\kpartition$, we have
	$$[1]+[2]-[3]<0$$
	which concludes the proof.
\end{proof}

\subsection{Proofs for Section~\ref{s:main2}}
We need the following two lemmas to prove Theorem~\ref{lowerbound}. For a linear transformation $T$ on a finite dimensional vector space, $\lambda_{\max}(T)$ and $\lambda_{\min}(T)$ denote the largest and the smallest eigenvalue of $T$, respectively.
\begin{lemma}\label{Tk}
	Let $T$ be an $n\times n$ matrix such that $\norminf{T}\leq 1$. Then
	\[\|x-T^{k}x\|_{\infty}\leq k\|x-Tx\|_{\infty},\]
	for all $x\in\mathbb{C}^{n}$ and $k\in\mathbb{N}$.
\end{lemma}
\begin{proof}
	We have
	\begin{eqnarray*}
		\|x-T^{n}x\|_{\infty}&\leq&\|x-Tx\|_{\infty}+\|Tx-T^{2}x\|_{\infty}+\ldots+\|T^{k-1}x-T^{k}x\|_{\infty}\\&=&
		\|x-Tx\|_{\infty}+\|T(x-Tx)\|_{\infty}+\ldots+\|T^{k-1}(x-Tx)\|_{\infty}\leq k\|x-Tx\|_{\infty},
	\end{eqnarray*}
	where the last inequality follows from $\norminf{T}\leq 1$.
\end{proof}
\begin{lemma}\label{differencelowerbound}
	Let $T$ be a self-adjoint $n\times n$ matrix, $n\geq 3$, such that $\norminf{T}\leq 1$. Let $\mathcal{M}$ be a subspace of $\mathbb{C}^{n}$ such that $T\mathcal{M}\subset\mathcal{M}$. Then
	\[\|x-Tx\|_{\infty}\geq\frac{(1-\lambda_{\max}(T|_{\mathcal{M}}))\|x\|_{\infty}}{2\ln n},\]
	for all $x\in\mathcal{M}$.
\end{lemma}
\begin{proof}
	We may assume that $T$ is positive semidefinite. Indeed, $\frac{I+T}{2}$ is positive semidefinite, and if the result holds with $T$ being replaced by $\frac{I+T}{2}$, the result will hold for $T$.
	
	Since $T$ is positive semidefinite,
	\[\|T^{k}x\|_{2}\leq\lambda_{\max}(T|_{\mathcal{M}})^{k}\|x\|_{2},\]
	for all $x\in\mathcal{M}$ and $k\in\mathbb{N}$. So $\|T^{k}x\|_{\infty}\leq\lambda_{\max}(T|_{\mathcal{M}})^{k}\sqrt{n}\|x\|_{\infty}$ and hence, by Lemma \ref{Tk},
	\[\|x\|_{\infty}-\lambda_{\max}(T|_{\mathcal{M}})^{k}\sqrt{n}\|x\|_{\infty}\leq k\|x-Tx\|_{\infty},\]
	for all $x\in\mathcal{M}$ and $k\in\mathbb{N}$. If $k$ is large enough so that $\lambda_{\max}(T|_{\mathcal{M}})^{k}\sqrt{n}\leq\frac{1}{2}$, then $\|x-Tx\|_{\infty}\geq\frac{1}{2k}\|x\|_{\infty}$. Since $\norminf{T}\leq 1$, we have $\lambda_{\max}(T|_{M})\leq 1$. Note that $\lambda\leq e^{\lambda-1}$ for all $\lambda\leq 1$. So $\lambda_{\max}(T|_{\mathcal{M}})^{k}\sqrt{n}\leq e^{k(\lambda_{\max}(T|_{\mathcal{M}})-1)}\sqrt{n}\leq\frac{1}{2}$ for $k\geq\frac{\ln(2\sqrt{n})}{1-\lambda_{\max}(T|_{\mathcal{M}})}$. Taking $k$ to be the smallest integer larger than or equal to $\frac{\ln(2\sqrt{n})}{1-\lambda_{\max}(T|_{\mathcal{M}})}$, we obtain
	\[\|x-Tx\|_{\infty}\geq\frac{1}{2k}\|x\|_{\infty}\geq\frac{(1-\lambda_{\max}(T|_{\mathcal{M}}))\|x\|_{\infty}}{2\ln n},\]
	if $n\geq 30$. If $3\leq n\leq 29$, then
	\[\|x-Tx\|_{\infty}\geq\frac{1}{\sqrt{n}}\|x-Tx\|_{2}\geq\frac{1}{\sqrt{n}}(1-\lambda_{\max}(T|_{\mathcal{M}}))\|x\|_{2}\geq\frac{(1-\lambda_{\max}(T|_{\mathcal{M}}))\|x\|_{\infty}}{2\ln n},\]
	for all $x\in\mathcal{M}$.
\end{proof}
\begin{proof}[\bf Proof of Theorem~\ref{lowerbound}]
	Without loss of generality, we may assume that $B_{i,i}\leq 1$ for all $1\leq i\leq n$. For every $1\leq i\leq n$,
	\[\sum_{j=1}^{n}|(I-B)_{i,j}|=1-B_{i,i}+\sum_{j\in\{1,\ldots,n\}\backslash\{i\}}|B_{i,j}|\leq1-B_{i,i}+B_{i,i}\leq 1.\]
	So $\norminf{I-B}\leq 1$. By Lemma \ref{differencelowerbound}, we have
	\[\|Bx\|_{\infty}=\|x-(I-B)x\|_{\infty}\geq\frac{(1-\lambda_{\max}((I-B)|_{\mathcal{M}}))\|x\|_{\infty}}{2\ln n}=\frac{\lambda_{\min}(B|_{\mathcal{M}})\|x\|_{\infty}}{2\ln n},\]
	for all $x\in\mathcal{M}$.
\end{proof}
\begin{proof}[\bf Proof of Corollary~\ref{graph}]
	Since $L\ones_n=0$ and $L$ is self-adjoint, $L\{\ones_n\}^{\perp}\subset\{\ones_n\}^{\perp}$. By Theorem~\ref{lowerbound},
	\[\|Lx\|_{\infty}\geq\frac{\lambda_{\min}\lp L|_{\{\ones_n\}^{\perp}}\rp\|x\|_{\infty}}{2\ln n}=\frac{\lambda_{2}(L)\|x\|_{\infty}}{2\ln n},\]
	for all $x\perp \ones_n$. This proves one inequality.
	
	To prove the other inequality, pick a vertex $u_{0}$. Let $y\in\mathbb{C}^{n}$ be given by $y(v)=d(u_{0},v)$, for vertices $v$, where $d$ is the graph distance. Then
	\[(Ly)(v)=\mathrm{deg}(v)d(u_{0},v)-\sum_{w\in N(v)}d(u_{0},w),\]
	for all vertex $v$, where $N(v)$ is the set of all neighborhood vertices of $v$. Since $|d(u_{0},v)-d(u_{0},w)|\leq 1$ for all $w\in N(v)$, we have $|(Ly)(v)|\leq\mathrm{deg}(v)$ for all vertex $v$. So $\|Ly\|_{\infty}\leq M$.
	
	Let $z=y-(\frac{1}{n}\sum_{v}y(v))\ones_n\in\mathbb{C}^{n}$, where the sum is over all vertices $v$. It is easy to see that there exists a vertex $w$ such that $d(u_{0},w)\geq\frac{D}{2}$, where $D$ is the diameter. So $|z(w)-z(u_{0})|=|y(w)-y(u_{0})|=|d(u_{0},w)-0|=d(u_{0},w)\geq\frac{D}{2}$. Thus, $\|z\|_{\infty}\geq\frac{D}{4}$. Since $L\ones_n=0$, we have $\|Lz\|_{\infty}=\|Ly\|_{\infty}\leq M$. Therefore,
	\[\inf_{x\perp\ones_n}\frac{\|Lx\|_{\infty}}{\|x\|_{\infty}}\leq\frac{M}{D/4}=\frac{4M}{D}.\]
	The result follows.
\end{proof}
With the help of Corollary~\ref{graph}, we can prove Theorem~\ref{thm:twoinfbound}.
\begin{proof}[\bf Proof of Theorem~\ref{thm:twoinfbound}]
	We use the same notation as Lemma~\ref{lem:twoinfbound}. Note that $\normtwotoinf{\Uiso}=\max_{1\leq i\leq k}1/\sqrt{|V_i|}$ so $\mu=\sqrt{c}$. By Corollary~\ref{graph} and~\eqref{eq:gap} we have 
	$$\gap\geq\min_{1\leq i\leq k}\frac{\lambda_2(L_i)}{2\ln |V_i|}\geq\frac{\min_{1\leq i\leq k}\lambda_2(L_i)}{2\ln n}.$$
	If
	$$r = \frac{\max_{1\leq i\leq n}\ddi}{\min_{1\leq i\leq k}\lambda_2(L_i)}\leq\frac{1}{16(1+c)\ln n},$$
	then 
	$$\norminf{\Ld}=2\max_{1\leq i\leq n}\ddi\leq\frac{\gap}{4(1+\mu^2)}$$
	and
	$$\norm{\Ld}\leq\sqrt{\norminf{\Ld}||\Ld||_1}=\norminf{\Ld}=2\max_{1\leq i\leq n}\ddi\leq\frac{\gap}{5}.$$
	Therefore by Lemma~\ref{lem:twoinfbound} we have
	\begin{align*}
		\min_{V\in{\bf O}^k}\normtwotoinf{UV-\Uiso}&\leq 8\normtwotoinf{\Uiso}\lp\frac{\norm{\Ld}}{\sep_2(\Lambda_1,\Lambda_2)}\rp^2+4\frac{\normtwotoinf{U_2U_2^T\Ld\Uiso}}{\gap}\\
		&\leq 8\sqrt{\frac{c}{n}}\lp\frac{2\max_{1\leq i\leq n}\ddi}{\min_{1\leq i\leq k}\lambda_2(L_i)}\rp^2+\frac{8\ln n\normtwotoinf{U_2U_2^T\Ld\Uiso}}{\min_{1\leq i\leq k}\lambda_2(L_i)}\\
		&=32\sqrt{c}r^2\frac{1}{\sqrt{n}}+\frac{8\ln n\normtwotoinf{U_2U_2^T\Ld\Uiso}}{\min_{1\leq i\leq k}\lambda_2(L_i)},
	\end{align*}
	where we have used~\eqref{eq:sepdiag} in the second step.
	Finally 
	\begin{align*}
	\normtwotoinf{U_2U_2^T\Ld\Uiso}&=\normtwotoinf{(I-\Uiso\Uiso^T)\Ld\Uiso}\\
	 &=\normtwotoinf{\Ld\Uiso-\Uiso\Uiso^T\Ld\Uiso}\\
	 &\leq\lp\norminf{\Ld}+\norm{\Uiso^T\Ld\Uiso}\rp\normtwotoinf{\Uiso}\\
	 &\leq\lp\norminf{\Ld}+\norm{\Ld}\rp\normtwotoinf{\Uiso}\\
	 &\leq 4\sqrt{\frac{c}{n}}\max_{1\leq i\leq n}\ddi.
	\end{align*}
	Hence
	$$\min_{V\in{\bf O}^k}\normtwotoinf{UV-\Uiso}\leq32\sqrt{c}\lp r^2+r\ln n\rp\frac{1}{\sqrt{n}}.$$
\end{proof}

\subsection{A lemma for Section~\ref{s:discussion}}\label{s:lemma}
\begin{lemma}\label{lem:sep}
	Let $c_1,c_2\in\mathbb{R}^n$ such that $\norm{c_1-c_2}=d$. For any $x\in B_{c_1}(r)$ and $y\in B_{c_2}(r)$ let $M$ be the $(n-1)$-dimensional bisecting hyperplane that passes through $\frac{x+y}{2}$ and is perpendicular to the line segment that joins $x$ and $y$. Then 
	$$\dist\lc M, B_{c_1}(r)\cup B_{c_2}(r)\rc\geq \frac{1}{2}d-3r.$$
\end{lemma}
\begin{proof}
	By symmetry, it suffices to show $\dist\lc M, B_{c_1}(r)\rc\geq \frac{1}{2}d-3r$. We may suppose $c_1=0$ and $d>6r$. Then it suffices to show $\dist\lc M,0\rc\geq\frac{1}{2}d-2r$. For any $z\in M$ we have
	$$\sum_{i=1}^{n}(x_i-y_i)\lp z_i-\frac{x_i+y_i}{2}\rp=0.$$
	By the point-plane distance formula
	\begin{align*}
		\dist\lc M,0\rc=\frac{\abs{\sum_{i=1}^{n}(x_i^2-y_i^2)}}{2\sqrt{\sum_{i=1}^{n}(x_i-y_i)^2}}&=\frac{\norm{y}^2-\norm{x}^2}{2\norm{x-y}}\\
		&\geq\frac{(d-r)^2-r^2}{2(d+2r)}\\
		&=\frac{1}{2}d-\frac{2dr}{d+2r}\\
		&\geq\frac{1}{2}d-2r.
	\end{align*}
	
\end{proof}
\bibliography{spectral}
\bibliographystyle{abbrv}
\end{document}